%% file: collranking.tex
\documentclass[11pt]{article}

\usepackage{fullpage}
\usepackage{hyperref}
\usepackage{amsmath,amsfonts,amsthm,amssymb,xspace,bm,verbatim,dsfont,mathtools}
\usepackage{url,algorithm,algorithmic}
\usepackage{color}
\usepackage{epstopdf}
\usepackage[titletoc,title]{appendix}
\usepackage{comment}
\usepackage[shortlabels]{enumitem}

\usepackage{graphicx,caption,subcaption}
\usepackage[numbers]{natbib}

\input{./tex/preamble}

\begin{document}

\title{Preference Completion: Large-scale Collaborative Ranking from Pairwise Comparisons}
\author{
Dohyung Park\\
{The University of Texas at Austin}\\
{\href{mailto:dhpark@utexas.edu}{dhpark@utexas.edu}}
\and
Joe Neeman\\
{The University of Texas at Austin}\\
{\href{mailto:joeneeman@gmail.com}{joeneeman@gmail.com}}
\and
Jin Zhang\\
{The University of Texas at Austin}\\
{\href{mailto:zj@utexas.edu}{zj@utexas.edu}}
\and
Sujay Sanghavi\\
{The University of Texas at Austin}\\
{\href{mailto:sanghavi@mail.utexas.edu}{sanghavi@mail.utexas.edu}}
\and
Inderjit S. Dhillon\\
{The University of Texas at Austin}\\
{\href{mailto:inderjit@cs.utexas.edu}{inderjit@cs.utexas.edu}}
}
\maketitle

\begin{abstract} 
In this paper we consider the collaborative ranking setting: a pool of users each provides a small number of pairwise preferences between $d$ possible items; from these we need to predict each user’s preferences for items they have not yet seen. We do so by fitting a rank $r$ score matrix to the pairwise data, and provide two main contributions: \\
{\em (a)} we show that an algorithm based on convex optimization provides good generalization guarantees once each user provides as few as $O(r\log^2 d)$ pairwise comparisons -- essentially matching the sample complexity required in the related matrix completion setting (which uses actual numerical as opposed to pairwise information), and \\
{\em (b)} we develop a large-scale non-convex implementation, which we call AltSVM, that trains a factored form of the matrix via alternating minimization (which we show reduces to alternating SVM problems), and scales and parallelizes very well to large problem settings. It also outperforms common baselines on many moderately large popular collaborative filtering datasets in both NDCG and in other measures of ranking performance. 
\end{abstract} 

\input{./tex/intr.tex}

\input{./tex/form.tex}
\input{./tex/stat.tex}
\input{./tex/algo.tex}
\input{./tex/expr.tex}

\input{./tex/disc.tex}

\bibliography{collranking}
\bibliographystyle{icml2015}

\appendix
\input{./tex/upper-proof.tex}

\input{./tex/matrix-proof.tex}
\input{./tex/lower-sketch.tex}
\input{./tex/lower-proof.tex}

\input{./tex/sgd.tex}

\end{document}

%% file: tex/preamble.tex
\usepackage{amsmath}
\usepackage{amsfonts}
\usepackage{amsthm}
\usepackage{amssymb}
\usepackage{latexsym,xspace,float,bm}

\newcommand{\E}{\mathbb{E}}

\renewcommand{\P}{\mathbb{P}}

\newcommand{\calA}{\mathcal{A}}

\newcommand{\calX}{\mathcal{X}}

\newcommand{\pdiffII}[3]{\ifstrequal{#2}{#3}
{\frac{\partial^2 #1}{\partial #2^2}}
{\frac{\partial^2 #1}{\partial #2 \partial #3}}
}
\newcommand{\diffII}[3]{\ifthenelse{\equal{#2}{#3}}
{\frac{d^2 #1}{d #2^2}}
{\frac{d^2 #1}{d #2 d #3}}
}

\let\Pr\relax
\DeclareMathOperator{\Pr}{Pr}

\DeclareMathOperator{\rank}{rank}

\DeclareMathOperator{\tr}{tr}

\newtheorem{theorem}{Theorem}[section]
\newtheorem{assumption}{Assumption}[section]

\newtheorem{lemma}[theorem]{Lemma}
\newtheorem{corollary}[theorem]{Corollary}
\newtheorem{proposition}[theorem]{Proposition}

\newtheorem{definition}[theorem]{Definition}



%% file: tex/intr.tex
\section{Introduction}

This paper considers the following recommendation system problem: given a set of items, a set of users, and non-numerical {\em pairwise comparison} data, find the underlying preference ordering of the users. In particular, we are interested in the setting where data is of the form ``user $i$ preferes item $j$ over item $k$", for different ordered user-item-item triples $i,j,k$. Pairwise preference data is wide-spread; indeed, almost any setting where a user is presented with a menu of options -- and chooses one of them -- can be considered to be providing a pairwise preference between the chosen item and every other item that is presented. 

Crucially, we are interested in the collaborative filtering setting, where {\em (a)} on the one hand the number of such pairwise preferences we have for any one user is woefully insufficient to infer anything
for that user in isolation; and {\em (b)} on the other hand, we aim for {\em personalization}, i.e. for every user to possibly have different inferred preferences from every other. To reconcile these two requirements, our method relates the preferences of users to each other via a low-rank matrix, which we (implicitly) assume governs the observed preferences. Essentially, we fit a low-rank users $\times$ items {\em score matrix} $X$ to pairwise comparison data by trying to ensure that $X_{ij} - X_{ik}$ is positive
when user $i$ prefers item $j$ to item $k$.

\paragraph{Our contributions:} We present two algorithms to infer the score matrix $X$ from training data; once inferred, this can be used for predicting future preferences. While there has been some recent work on fitting low-rank score matrices to pairwise preference data (which we review and compare to below), in this paper we present the following two contributions:\\
{\bf (a)} {\em A statistical analysis for the convex relaxation:} we bound the \emph{generalization error} of the solution to our convex program. Essentially, we show that the minimizer of the empirical loss also almost minimizes the true expected loss. We also give a lower bound showing that our error rate is sharp
up to logarithmic factors. \\
%
{\bf (b)} {\em A large-scale non-convex implementation:} 
We provide a non-convex algorithm that we call Alternating Support Vector Machine (AltSVM). This non-convex algorithm is more practical than the convex program in a large-scale setting; it explicitly parameterizes the low-rank matrix in factored form and minimizes the hinge loss. Crucially, each step in this algorithm can be formulated as a standard SVM that updates one of the two factors; the algorithm proceeds by alternating updates to both factors. We apply a stochastic version of dual coordinate descent \cite{DualSVM,SDCA} with lock-free parallelization. This exploits the problem structure and ensures it parallelizes well. We show that our algorithm outperforms several existing collaborative ranking algorithms in both speed and prediction accuracy, and it achieves significant speedups as the number of cores increases.

\subsection{Related Work}
Ranking/learning preferences is a classical problem that has been considered in a large amount of work. There are many different settings for this problem, which we discuss below.

\paragraph{Learning to Rank}
The main problem in this community has been to estimate a ranking function from given feature vectors and relevance scores. Depending on its application, a feature vector may correpond to a user-item pair or a single item. While there have been algorithms that use pairwise comparisons \cite{HerbrichOrdinal,JoachimsRankSVM} of the training samples, our setting is different in that our data consists \emph{only} of pairwise comparisons. We refer the reader to the survey \cite{LiuLETOR:10}.

\paragraph{One ranking with pairwise comparisions}
In a single-user model, we are asked to learn a single ranking given pairwise comparisons. \citet{JamiesonNowak:11} and \citet{Ailon:11} consider an active query model with noiseless responses; \citet{JamiesonNowak} give an algorithm for exactly recovering the true ranking under a low-rank assumption similar to ours, while \citet{Ailon:11} approximately
recovers the true ranking without such an assumption. \citet{WaJoJo:13} and \citet{NeOhSh:12} learn a ranking from noisy pairwise comparisions; \citet{NeOhSh:12} consider a Bradley-Terry-Luce model similar to ours and attempt to learn an
underlying score vector, while \citet{WaJoJo:13} get by without structure assumptions, but only attempt to learn the ranking itself. \citet{HajekOhXu:14} considered a problem to learn a single ranking given a more generalized partial rankings from the Plackett-Luce model and provided a minimax-optimal algorithm.

\paragraph{Many rankings with pairwise comparisions}
Given multiple users with different rankings, one could of course attempt to learn their rankings by simply applying an algorithm from the previous section to each user individually. However, it is more efficient -- both statistically and computationally -- to postulate some global structure and use it to relate the many users' rankings. This is the same idea that has been applied so successfully in collaborative filtering. \citet{RFGS:09} and \citet{LiZhYa:09} were the first to take this approach. They modeled the observations as coming from a BTL model with low-rank structure (i.e., very similar to our model) and gave algorithms for learning the model parameters. \citet{YiJiJaJa:13} took a purely optimization-based approach. Rather than assuming a probabilistic model, they minimized a convex objective using the hinge loss on a low-rank matrix. In a slightly different model, \citet{HuKoVo:08} and \citet{SKBLOH:13} consider the problem of learning from latent feedback. Recently, \citet{LuNegahban} analyzed an algorithm which is very similar to ours for the Bradley-Terry-Luce model independently from our work.

\paragraph{Many rankings with 1-bit ratings}
Instead of moving to pairwise comparisons, some work has suggested avoiding the difficulties of numerical ratings by instead asking users to give 1-bit ratings to items; that is, each user only indicates whether they like or dislike an item. In this setting, the work of \citet{DPVW:13} is most closely related to ours, in that they assume an underlying low-rank structure and give an algorithm based on convex optimization. Also, our theoretical analysis owes a lot to their work. \citet{XWZHSY:14} consider a slightly different goal: rather than attempting to recover the preferences of each user, they try to cluster similar users and similar items together. \citet{RobiRank} proposed an optimization problem motivated from robust binary classification and used stochastic gradient descent to solve the problem in a large-scale setting.

\paragraph{Many rankings with numerical ratings}
The goal in this setting is the same as ours, except that the data is in the form of numerical ratings instead of pairwise comparisons. \citet{CofiRank} attempted to directly optimize Normalized Discounted Cumulative Gain (NDCG), a widely used performance measure for ranking problems. \citet{BalakrishnanChopra}, and \citet{VolkovsZemel} converted this problem into a learning-to-rank problem and solved it using the existing algorithms. While these works considered the low-rank matrix model, different models are proposed by \citet{WestonWWB:12} and \citet{LocalColRanking}. \citet{WestonWWB:12} proposed a tensor model to rank items for different queries and users, and \cite{LocalColRanking} proposed a weighted sum of low-rank matrix models.

%% file: tex/form.tex
\section{Empirical Risk Minimization (ERM)}

Let us first formulate the problem mathematically. The task is to estimate rankings of multiple users on multiple items. We denote the numbers of users by $d_1$, and the number of items by $d_2$. We are given a set of triples $\Omega \subset [d_1] \times [d_2] \times [d_2]$, where the preference of user $i$ between items $j$ and $k$ is observed if $(i,j,k) \in \Omega$. The observed comparison is then given by $\{Y_{ijk} \in \{1,-1\} : (i,j,k) \in \Omega\}$ where $Y_{ijk} = 1$ if user $i$ prefers item $j$ over item $k$, and $Y_{ijk} = -1$ otherwise. Let $\Omega_i = \{(j,k) : (i,j,k) \in \Omega\}$ denote the set of item pairs that user $i$ has compared.

We predict rankings for multiple users by estimating a score matrix $X \in \mathbb{R}^{d_1 \times d_2}$ such that $X_{ij} > X_{ik}$ means that user $i$ prefers item $j$ over item $k$. Then the sorting order for each row provides the predicted ranking for the corresponding user.

We propose (as have others) that $X$ is low-rank or close to low-rank, the intuition being that each user bases their preferences on a small set of features that are common among all the items. Then the empirical risk minimization (ERM) framework can naturally be formulated as 
\begin{align} \label{eqn:formulation}
\underset{X}{\text{minimize}} &\quad \sum_{(i,j,k) \in \Omega} \mathcal{L}(Y_{ijk}(X_{ij} - X_{ik})) \\
\text{subject to} &\quad \mathrm{rank}(X) \le r \nonumber
\end{align}
where $\mathcal{L}(\cdot)$ is a monotonically non-increasing loss function which induces $X_{ij} > X_{ik}$ if $Y_{ijk} = 1$, and $X_{ij} < X_{ik}$ otherwise. (e.g., hinge loss, logistic regression loss, etc.)

Solving \eqref{eqn:formulation} is NP-hard because of the rank constraint. As a first alternative, we propose a straightforward convex relaxation.

%% file: tex/stat.tex
\section{Convex Relaxation}

Our first method is the convex relaxation of \eqref{eqn:formulation}, which involves a nuclear norm constraint.
\begin{align} \label{eqn:convex}
\underset{X}{\text{minimize}} &\quad \sum_{(i,j,k) \in \Omega} \mathcal{L}(Y_{ijk}(X_{ij} - X_{ik})) \\
\text{subject to} &\quad \|X\|_* \le \sqrt{\lambda d_1 d_2} \nonumber
\end{align}

Here, for any matrix $X$, the nuclear/trace norm $\|X\|_*$ denotes the sum of its singular values; it is a well-recognized convex surrogate for low-rank structure (most famously in matrix completion). 

The only parameter of this algorithm is $\lambda$, which governs the trade-off between better optimizing the likelihood of the observed data, and the strictness in imposing approximate low-rank structure.
Since we motivated our algorithm with the assumption that $X$ has low rank, we should point out how our algorithm's parameter $\lambda$ compares to the rank: note that if $X$ is a $d_1 \times d_2$ rank-$r$ matrix whose largest absolute entry is bounded by $C$ then $\|X\|_* \le \sqrt r \|X\|_F \le C \sqrt{r d_1 d_2}$. In other words, $\lambda$ is a parameter that takes into account both the rank of $X$ and the size of its elements, and it is roughly proportional to the rank.

\subsection{Analytic results}
We analyze \eqref{eqn:convex} by assuming a standard model for pairwise comparisons. Then we provide a statistical guarantee of the method under the model.

Recall the classical Bradley-Terry-Luce model~\cite{BradleyTerry:52,Luce:59} for pairwise preferences of a single user, which assumes that the probability of item $j$ being preferred over $k$ is given by a logistic of the difference of the underlying preference scores of the two items. For multiple users, we assume that there is some true score matrix $X^* \in \mathbb{R}^{d_1 \times d_2}$ and 
\[
\Pr(Y_{ijk} = 1) ~ = ~ \frac{\exp(X^*_{ij} - X^*_{ik})}{1 + \exp(X^*_{ij} - X^*_{ik})}.
\]

Assume that each user-item-item triple $(i, j, k)$ independently belongs to $\Omega$ with probability $p_{i,j,k}$, and let $m = \sum_{i,j,k} p_{i,j,k}$ be the expected size of $\Omega$. We will assume that the $p_{i,j,k}$ are approximately balanced in the sense that no user-item pair is observed too frequently:
\begin{assumption}\label{ass:balance}
 There is a constant $\kappa > 0$ such that for every $i, j$,
 \[
  \sum_k p_{i,j,k} \le \kappa \frac{m}{d_1 d_2}.
 \]
\end{assumption}
Note that if $\kappa = 1$ in Assumption~\ref{ass:balance} then the $p_{i,j,k}$ are all equal, meaning that each user-item-item triple has an equal chance to be observed.

In order to state our error bounds, we first introduce some notation: let $\P_X$ be the distribution of $\{Y_{i,j,k} : 1 \le i \le d_1, 1 \le j < k \le d_2\}$ (i.e.\ the complete distribution of all pairwise preferences, even those that are not observed).

Our main upper bound shows that if $m$ is sufficiently large then our algorithm finds a solution
with almost minimal risk. Given a loss function $\mathcal{L}$, define the expected risk
of $X$ by
\[
 R(X) = \frac{1}{d_1 d_2^2} \sum_{i=1}^{d_1} \sum_{j,k =1}^{d_2} \E_{X^*} \mathcal{L}(Y_{ijk} (X_{ij} - X_{ik})),
\]
where the expectation is with respect to the distribution parametrized by the true parameters $X^*$.

\begin{theorem}\label{thm:upper}
Suppose that $\mathcal{L}$ is 1-Lipschitz, and
let $Y$ and $\Omega$ be distributed as $\P_{X^*}$ for some $d_1 \times d_2$ matrix $X^*$.
Under Assumption~\ref{ass:balance},
\begin{align*}
 \E R(\hat X) \le \inf_{\{X: \|X\|_* \le \sqrt{\lambda d_1 d_2}\}} \E R(X)
 + C \kappa \sqrt{\frac{\lambda (d_1 + d_2)}{m}} \log (d_1 + d_2),
\end{align*}
where $C$ is a universal constant.
\end{theorem}

We recall that the parameter $\lambda$ is related to rank
in that if $X$ is a
$d_1 \times d_2$ rank-$r$ matrix whose largest absolute entry is bounded by $C$
then $\|X\|_* \le \sqrt r \|X\|_F \le C \sqrt{r d_1 d_2}$. In other words, $\lambda$ is a parameter
that takes into account both the rank of $X^*$ and the size of its elements, and
it is roughly proportional to the rank. In particular, Theorem~\ref{thm:upper} shows that
once we observe
$m \sim r (d_1 + d_2) \log^2(d_1 + d_2)$ pairwise comparisons, then we can accurately
estimate the probability of any user preferring any item over any other. In other words,
we need to observe about $r (1 + d_2/d_1) \log^2(d_1 + d_2)$ comparisons per user, which is
substantially less than the $r d_2 \log(d_2)$ comparisons that we would have required if
each user were modelled in isolation. Moreover, our lower bound (below) shows that
at least $r (1 + d_2/d_1)$ comparisons per user are required, which is only a logarithmic factor
from the upper bound.


\begin{theorem}\label{thm:lower}
Suppose that $\mathcal{L}'(0) < 0$.
Let $\calA$ be any algorithm that receives $\{Y_{i,j,k}: (i,j,k) \in \Omega\}$ as input
and produces $\hat X$ as output. For any $\lambda \ge 1$ and
$m \ge d_1 + d_2$, there exists $X^*$ with $\|X^*\|_* \le \sqrt{\lambda d_1 d_2}$ such that
when $Y$ and $\Omega$ are distributed according to $\P_{X^*}$ then with probability at least $\frac{1}{2}$,
\[
 \E R(\hat X) \ge R(X^*) + c\min\left\{1, \sqrt{\frac{\lambda (d_1 + d_2)}{m}}\right\},
\]
where $c > 0$ is a constant depending only on $\mathcal{L}$.
\end{theorem}

Together, Theorems~\ref{thm:upper} and~\ref{thm:lower} show that (up to logarithmic factors)
if $X^*$ has rank $r$ then about $r (1 + d_2 / d_1)$ comparisons per user
are necessary and sufficient for learning the users' preferences.

\subsubsection{Maximum likelihood estimation of $X^*$}

By specializing the loss function $\mathcal{L}$, Theorem~\ref{thm:upper}
has a simple corollary for maximum-likelihood estimation of $X^*$.
Recall that if $\mu$ and $\nu$ are two probability distributions on a finite set $S$
the the Kullback-Leibler divergence between them is
\[
D(\mu \| \nu) = \sum_{s \in S} \mu(s) \log\frac{\mu(s)}{\nu(s)},
\]
under the convention that $0 \log 0 = 0$. We recall that although $D(\cdot \| \cdot)$
is not a metric it is always non-negative, and that $D(\mu \| \nu) = 0$ implies
$\mu = \nu$.

\begin{corollary}\label{cor:upper}
Let $Y$ and $\Omega$ be distributed as $\P_{X^*}$ for some $d_1 \times d_2$ matrix $X^*$.
Define the loss function $\mathcal{L}$ by
$\mathcal{L}(z) = \log(1 + \exp(z)) - z$.
Under Assumption~\ref{ass:balance},
\begin{align*}
\frac{1}{d_1 d_2^2} \sup_{\{X: \|X\|_* \le \sqrt{\lambda d_1 d_2}\}} D(\P_{X^*} \| \P_{\hat X}) - D(\P_{X^*} \| \P_{X})
\le C \kappa \sqrt{\frac{\lambda (d_1 + d_2)}{m}} \log (d_1 + d_2),
\end{align*}
where $C$ is a universal constant.
\end{corollary}

Note that the loss function in Corollary~\ref{cor:upper} is exactly the negative logarithm
of the logistic function, and so $\hat X$ in Corollary~\ref{cor:upper} is the maximum-likelihood
estimate for $X^*$. Thus, Corollary~\ref{cor:upper} shows that the distribution
induced by the maximum-likelihood estimator is close to the true distribution in
Kullback-Leibler divergence.

%% file: tex/algo.tex

\section{Large-scale Non-convex Implementation}

While the convex relaxation is statistically near optimal, it is not ideal for large-scale datasets because it requires the solution of a convex program with $d_1 \times d_2$ variables. In this section we develop a non-convex variant which both scales and parallelizes very well, and has better empirical performance as compared to several existing empirical baseline methods.

Our approach is based on the following steps:
\begin{enumerate}
\item We represent the low-rank matrix in explicit factored form $X=UV^\top$ and replace the regularizer appropriately. This results in a non-convex optimization problem in $U \in \mathbb{R}^{d_1 \times r}$ and $V \in \mathbb{R}^{d_2 \times r}$, where $r$ is the rank parameter.
\item We solve the non-convex problem by alternating between updating $U$ while keeping $V$ fixed, and vice versa. With the hinge loss (which we found works best in experiments), each of these becomes an SVM problem - hence we call our algorithm AltSVM.
\item The problem is of course not symmetric in $U$ and $V$ because users rank items but not vice versa. For the $U$ update, each user vector naturally decouples and can be done in parallel (and in fact just reduces to the case of rankSVM \cite{JoachimsRankSVM}).
\item For the $V$ update, we show that this can {\em also} be made into an SVM problem; however it involves coupling of all item vectors, and all user ratings. We employ several tricks (detailed below) to speed up and effectively parallelize this step.
\end{enumerate}

The non-convex problem can be written as
\begin{align} \label{eqn:formulationUV}
\min_{U, V} ~ \sum_{(i,j,k) \in \Omega} \mathcal{L}(Y_{ijk} \cdot u_i^\top (v_j - v_k)) + \frac{\lambda}{2} (\|U\|_F^2 + \|V\|_F^2)
\end{align}
where we replace the nuclear norm regularizer using the property $\|X\|_* = \min_{X = U V^\top} \frac{1}{2}(\|U\|_F^2 + \|V\|_F^2)$ \cite{MMMF}. $u_i^\top$ and $v_i^\top$ denote the $i$th rows of $U$ and $V$, respectively. While this is a non-convex algorithm for which it is hard to find the global optimum, it is computationally more efficient since only $(d_1+d_2)r$ variables are involved. We propose to use L2 hinge loss, i.e., $\mathcal{L}(x) = \max(0,1-x)^2$.

In the alternating minimization of \eqref{eqn:formulationUV}, the subproblem for $U$ is to solve
\begin{align}
U \leftarrow & \arg \min_{U \in \mathbb{R}^{d_1 \times r}} \sum_{(i,j,k) \in \Omega} \mathcal{L}(Y_{ijk} \cdot u_i^\top (v_j - v_j)) + \frac{\lambda}{2} \|U\|_F^2, \label{eqn:subproblemU}
\end{align}
while $V$ is fixed. This can be decomposed into $n$ independent problems for $u_i$'s where each solves for
\begin{align} \label{eqn:usermin}
u_i \leftarrow & \arg \min_{u \in \mathbb{R}^r} \frac{\lambda}{2} \|u\|_2^2 + \sum_{(j,k) \in \Omega_i} \mathcal{L}(Y_{ijk} \cdot u^\top (v_j - v_k) .
\end{align}
This part is in general a small-scale problem as the dimension is $r$, and the sample size is $|\Omega_i|$ for each user $i$.

On the other hand, solving for $V$ with fixed $U$ can be written as
\begin{align} \label{eqn:itemmin}
V \leftarrow & \arg \min_{V \in \mathbb{R}^{d_2 \times r}} \left\{ \frac{\lambda}{2} \|V\|_F^2 + \sum_{(i,j,k) \in \Omega} \mathcal{L}( \langle V, A^{(u,i,j)} \rangle ) \right\}
\end{align}
where $A^{(i,j,k)} \in \mathbb{R}^{d_2 \times r}$ is such that the $l$th row of $A^{(i,j,k)}$ is $Y_{ijk} \cdot u_i^\top$ if $l=j$, $-Y_{ijk} \cdot u_i^\top$ if $l=k$, and $0$ otherwise. It is a much larger SVM problem than \eqref{eqn:usermin} as the dimension is $d_2 r$ and the sample size is $|\Omega|$.

We note that the feature matrices $\{A^{(i,j,k)} : (i,j,k) \in \Omega\}$ are highly sparse since in each feature matrix only $2r$ out of the $d_2 r$ elements are nonzero. This motivates us to apply the stochastic dual coordinate descent algorithm \cite{DualSVM,SDCA}, which not only converges fast but also takes advantages of feature sparsity in linear SVMs. Each coordinate descent step takes $O(r)$ computation, and iterations over $|\Omega|$ coordinates provide linear convergence \cite{SDCA}.

Now we describe the dual problems of our two subproblems explicitly. Let $\alpha \in \mathbb{R}^{|\Omega_i|}$ denote the dual vector for \eqref{eqn:usermin}, in which each coordinate is denoted by $\alpha_{ijk}$ where $(j,k) \in \Omega_i$. Then the dual problem of \eqref{eqn:usermin} is to solve
\begin{align} \label{eqn:usermin_dual}
\min_{\alpha \in \mathbb{R}^{|\Omega_i|}, \alpha \ge 0} & \frac{1}{2} \left\| \sum_{(j,k) \in \Omega_i} \alpha_{ijk} Y_{ijk} (v_j - v_k) \right\|_2^2 + \frac{1}{\lambda} \sum_{(j,k) \in \Omega_i} \mathcal{L}^*(- \lambda \alpha_{ijk})	
\end{align}
where $\mathcal{L}^*(z)$ is the convex conjugate of $\mathcal{L}$. 
At each coordinate descent step for $\alpha_{ijk}$, we find the value of $\alpha_{ijk}$ minimizing \eqref{eqn:usermin_dual} while all the other variables are fixed. If we maintain $u_i = \sum_{(j,k) \in \Omega_i} \alpha_{ijk} Y_{ijk} (v_j - v_k)$, then the coordinate descent step is simply to find $\delta^*$ minimizing
\begin{align} \label{eqn:user_dcd}
\frac{1}{2} \left\| u_i  + \delta^* Y_{ijk} (v_j - v_k) \right\|_2^2 + \frac{1}{\lambda} \mathcal{L}^*(- \lambda (\alpha_{ijk} + \delta^*))
\end{align}
and update $\alpha_{ijk} \leftarrow \alpha_{ijk} + \delta^*$.

The dual problem of \eqref{eqn:itemmin} is to solve
\begin{align} \label{eqn:itemmin_dual}
\min_{\beta \in \mathbb{R}^{|\Omega|}, \beta \ge 0} & \frac{1}{2} \left\| \sum_{(i,j,k) \in \Omega} \beta_{ijk} A^{(i,j,k)} \right\|_F^2 + \frac{1}{\lambda} \sum_{(i,j,k) \in \Omega} \mathcal{L}^*(-\lambda\beta_{ijk})
\end{align}
where $\beta$ is the dual vector for the subproblem \eqref{eqn:itemmin}. Similarly to $\alpha_{ijk}$, the coordinate descent step for $\beta_{ijk}$ is to replace $\beta_{ijk}$ by $\beta_{ijk} + \delta^*$ where $\delta^*$ minimizes
\begin{align} \label{eqn:item_dcd}
\frac{1}{2} \left( \left\| v_j  + \delta^* Y_{ijk} u_i \right\|_2^2 + \left\| v_k - \delta^* Y_{ijk} u_i \right\|_2^2 \right) + \mathcal{L}^*(-\lambda(\beta_{ijk} + \delta^*)),
\end{align}
and maintain $V = \sum_{(i,j,k) \in \Omega} \beta_{ijk} Y_{ijk} A^{(i,j,k)}$. 

The detailed description of AltSVM is presented in Algorithm \ref{alg:altsvm_stoc}. In each subproblem, we run the stochastic dual coordinate descent, in which a pairwise comparison $(i,j,k) \in \Omega$ is chosen uniformly at random, and the dual coordinate descent for $\alpha_{ijk}$ or $\beta_{ijk}$ is computed. We note that each coordinate descent step takes the same $O(r)$ computational cost in both subproblems, while the subproblem sizes are much different.

\begin{algorithm}[tb]
	\caption{Alternating Support Vector Machine (AltSVM)}
	\label{alg:altsvm_stoc}
\begin{algorithmic}[1]
\REQUIRE $\Omega$, $\{Y_{ijk} : (i,j,k) \in \Omega\}$, and $\lambda \in \mathbb{R}^+$
\ENSURE $U \in \mathbb{R}^{d_1 \times r}$, $V \in \mathbb{R}^{d_2 \times r}$
\STATE Initialize $U$, and set $\alpha, \beta \leftarrow 0 \in \mathbb{R}^{|\Omega|}$
\WHILE {not converged}
	\STATE $v_j \leftarrow \sum_{(i,j,k) \in \Omega} \beta_{ijk} Y_{ijk} u_i$ \\ \qquad \quad $- \sum_{(i,k,j) \in \Omega} \beta_{ikj} Y_{ikj} u_i$, $\forall j \in [d_2]$
	\FOR {\textbf{all threads} $t=1,\ldots,T$ \textbf{in parallel}}
		\FOR {$s = 1, \ldots, S$}
			\STATE Choose $(i,j,k) \in \Omega$ uniformly at random
			\STATE Find $\delta^*$ minimizing \eqref{eqn:item_dcd}.
			\STATE $\beta_{ijk} \leftarrow \beta_{ijk} + \delta^*$
			\STATE $v_j \leftarrow v_j + \delta^* Y_{ijk} u_i$
			\STATE $v_k \leftarrow v_k - \delta^* Y_{ijk} u_i$
		\ENDFOR
	\ENDFOR
	\STATE $u_i \leftarrow \sum_{(i,j,k) \in \Omega} \alpha_{ijk} Y_{ijk} (v_j - v_k)$, $\forall i \in [d_1]$
	\FOR {\textbf{all threads} $t=1,\ldots,T$ \textbf{in parallel}}
		\FOR {$s = 1, \ldots, S$}
			\STATE Choose $(i,j,k) \in \Omega$ uniformly at random.
			\STATE Find $\delta^*$ minimizing \eqref{eqn:user_dcd}.
			\STATE $\alpha_{ijk} \leftarrow \alpha_{ijk} + \delta^*$
			\STATE $u_i \leftarrow u_i + \delta^* Y_{ijk} (v_j - v_k)$
		\ENDFOR
	\ENDFOR
\ENDWHILE
\end{algorithmic}
\end{algorithm}

\subsection{Parallelization}
For each subproblem, we parallelize the stochastic dual coordinate descent algorithm asynchronously without locking. Given $T$ processors, each processor randomly sample a triple $(i,j,k) \in \Omega$ and update the corresponding dual variable and the user or item vectors. We note that this update is for a sparse subset of the parameters. In the user part, a coordinate descent step for one sample updates only $r$ out of the $rd_1$ variables. In the item part, one coordinate descent step for a sample update only $2r$ out of the $rd_2$ variables. This motivates us not to lock the variables when updated, so that we ignore the conflicts. This lock-free parallelism is shown to be effective in \cite{Hogwild} for stochastic gradient descent (SGD) on the sum of sparse functions. Moreover, in \cite{StochasticDCD}, it is also shown that the stochastic dual coordinate descent scales well without locking. We implemented the algorithm using the OpenMP framework. In our implementations, we also parallelized steps 3 and 13 of Algorithm \ref{alg:altsvm_stoc}. We show in the next section that our proposed algorithm scales up favorably.

\subsection{Remark on the implementation}
In Algorithm \ref{alg:altsvm_stoc}, the subproblem for $V$ comes first, and then it solves for the user vectors $U$. We empirically observed that this order gives better convergence on practical datasets. We also note that each subproblem reuses the dual variables in the previous outer iteration. When almost converged, the features ($V$ for solving $U$, and $U$ for solving $V$) do not change too much. By reusing the dual variables in the previous iteration we can start with a feasible solution close to the optimum. 

%% file: tex/expr.tex

\def\B{\textbf}

\section{Experimental results}

\subsection{Pairwise data}
We used the MovieLens 100k dataset, which contains 100,000 ratings given by 943 users on 1682 movies. The ratings are given as integers from one to five, but we converted them into preference data by declaring that a user preferred one movie to another if they gave it a higher rating (if two movies received the same rating, we treated it as though the user did not provide a preference). Then we held out $20\%$ of the data as a test set.

We compared our algorithm to the following two:
\begin{itemize}
\item Bayesian Personalized Ranking (BPR) \cite{RFGS:09}: This algorithm is based on a similar model to ours,
but a different optimization procedure (essentially, a variant of stochastic gradient descent).

\item Matrix completion from pairwise differences : A standard matrix completion algorithm that observes -- for various triples $(i,j,k) \in \Omega$ -- the difference between user $i$'s ratings for item $j$ and item $k$. Note that this algorithm has an advantage over~\eqref{eqn:convex} because it sees the magnitude of this difference instead of only its sign. Nevertheless, the matrix completion algorithm does not perform any better than~\eqref{eqn:convex}. A similar phenomenon was also observed in~\cite{DPVW:13}.
\end{itemize}

We evaluate our performance by computing the proportion of pairwise comparisons in the test set $\mathcal{T}$ for which we correctly infer the user's preference.
$$
(\text{Prediction error}) = \frac{1}{|\mathcal{T}|} \sum_{(i,j,k) \in \mathcal{T}, Y_{ijk} = 1} \mathbb{I}( X_{ij} > X_{ik})
$$
This is similar to the AUC statistic measured by Rendle et al.~\cite{RFGS:09}, and if the data were fully observed then it would measure Kendall's distance between each user's true preferences and the learned ones. However, our main reason for choosing this measure of performance is that, as an average accuracy over all pairwise comparisions, it resembles the quantity that we study in our theoretical bounds.

Unsurprisingly, we were more accurate at correctly inferring strong preferences; therefore, we have also shown the accuracy obtained by only measuring performance on pairs whose rankings differ by two or more. Both the methods we considered do measurably better at predicting these orderings.

\begin{figure}
\begin{center}
\includegraphics[width=0.7\textwidth]{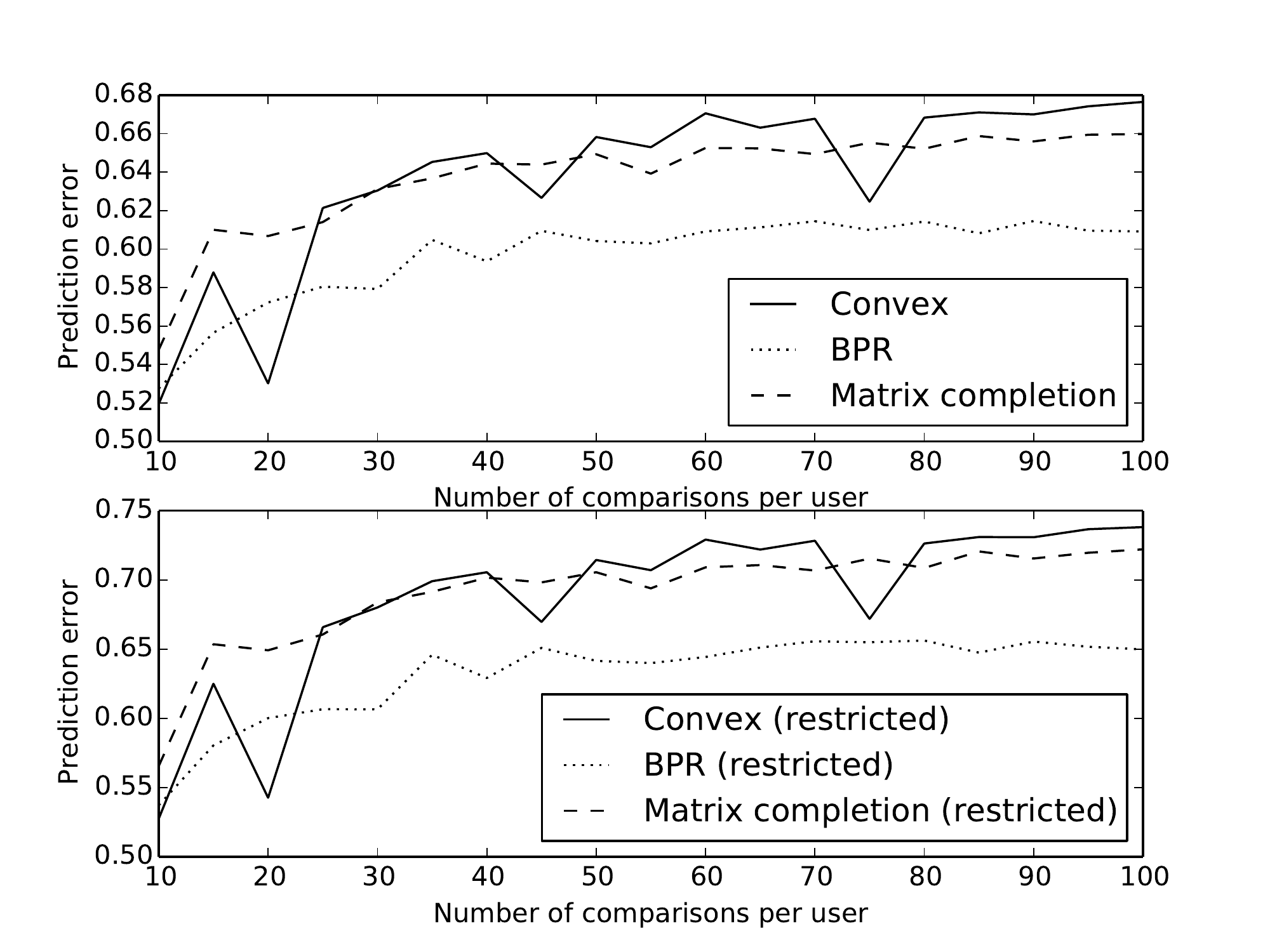}
\caption{Prediction accuracy on the MovieLens 100k dataset, for different numbers of observed
comparisions per user. For the ``restricted'' plots, only the pairs with a rating difference of two
or more were used for evaluation.}
\label{fig:ml-100k}
\end{center}
\end{figure}

\subsection{Large-scale experiments on rating data}
Now we demonstrate that our algorithm performs well as a collaborative ranking method on rating data. We used the datasets specified in Table \ref{tab:datasets}. Given a training set of ratings for each user, our algorithm will only use non-tying pairwise comparisons from the set, while other competing algorithms use the ratings themselves. Hence, they have more information than ours. The competing algorithms are those with publicly available codes provided by the authors.

\begin{itemize}
\item CofiRank \cite{CofiRank}\footnote{\url{http://www.cofirank.org}, The dimension and the regularization parameter are set as suggested in the paper. For the rest of the parameters, we left them as provided.} This algorithm uses alternating minimization to directly optimize NDCG.

\item Local Collaborative Ranking (LCR) \cite{LocalColRanking}\footnote{\url{http://prea.gatech.edu}, We run the code with each of the 48 sets of loss function and parameters given in the main code, and the best result is reported. We could not run this algorithm on the Netflix dataset due to time constraint.} : The main idea is to predict preferences from the weighted sum of multiple low-rank matrices model.

\item RobiRank \cite{RobiRank}\footnote{\url{https://bitbucket.org/d_ijk_stra/robirank}, We used the part for collaborative ranking from binary relevence score. We left the parameter settings as provide with the implementation.} : This algorithm uses stochastic gradient descent to optimize the loss function motivated from robust binary classification. 

\item Global Ranking : To see the effect of personalized ranking, we compare the results with a global ranking of the items. We fixed $U$ to all ones and solved for $V$.

\end{itemize}

\begin{table}[t]
\begin{center}
{\small
  \begin{tabular}{|c|ccc|}
\hline
        & MovieLens1m & MovieLens10m &     Netflix \\
\hline
Users   &       6,040 &       71,567 &     480,000 \\
Items   &       3,900 &       10,681 &      17,000 \\
Ratings &   1,000,209 &   10,000,054 & 100,000,000 \\
\hline
  \end{tabular}
}
  \caption{Datasets to be used for simulation} \label{tab:datasets}
\end{center}  
\end{table}

The algorithms are compared in terms of two standard performance measures of ranking, which are NDCG and Precision@$K$.
NDCG@$K$ is the ranking measure for numerical ratings. NDCG@$K$ for user $i$ is defined as
\begin{align*}
\mathrm{NDCG@}K(i) &= \frac{\mathrm{DCG@}K(i,\pi_i)}{\mathrm{DCG@}K(i,\pi^*_i)}
\end{align*}
where
\begin{align*}
\mathrm{DCG@}K(i,\pi_i) &= \sum_{k=1}^{K} \frac{2^{M_{i\pi_i(k)}} - 1}{\log_2 (k+1)},
\end{align*}
and $\pi_u(k)$ is the index of the $k$th ranked item of $\mathcal{T}_i$ in our prediction. $M_{ij}$ is the true rating of item $j$ by user $i$ in the given dataset, and $\pi_u^*$ is the permutation that maximizes DCG@$K$. This measure counts only the top $K$ items in our predicted ranking and put more weights on the prediction of highly ranked items. We measured NDCG@$10$ in our experiments. Precision@$K$ is the ranking measure for binary ratings. Precision@$K$ for user $i$ is defined as
\begin{align*}
\mathrm{Precision@}K(i) &= \frac{1}{K} \sum_{j \in \mathcal{P}_K(i)} M_{ij}
\end{align*}
where $M_{ij}$ is the binary rating on item $j$ by user $i$ given in the dataset. This counts the number of relevant items in the predicted top $K$ recommendation. These two measures are averaged over all of the users.

\begin{figure*}[ht]
\begin{center}
  \includegraphics[width=.32\textwidth]{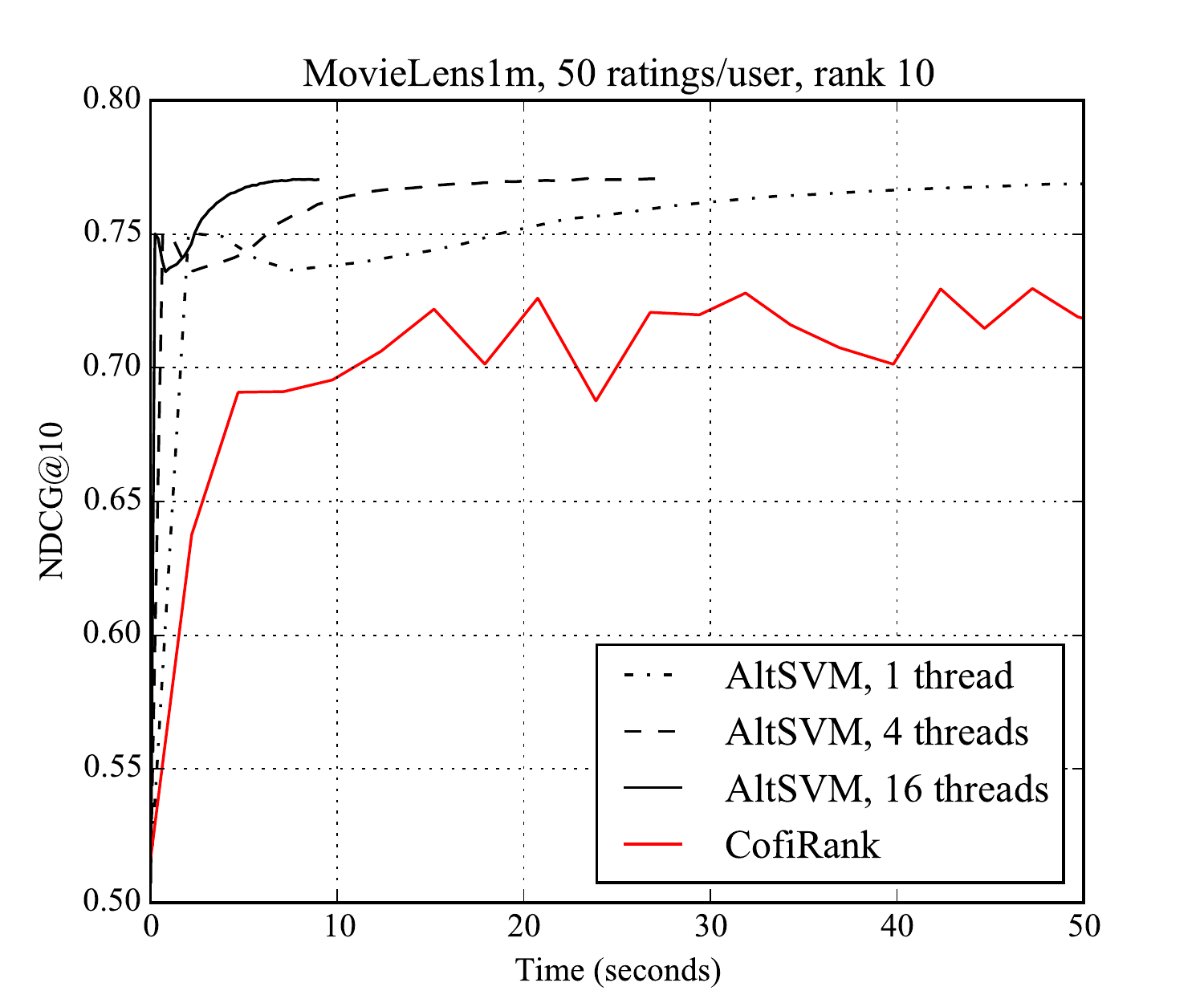}
  \includegraphics[width=.32\textwidth]{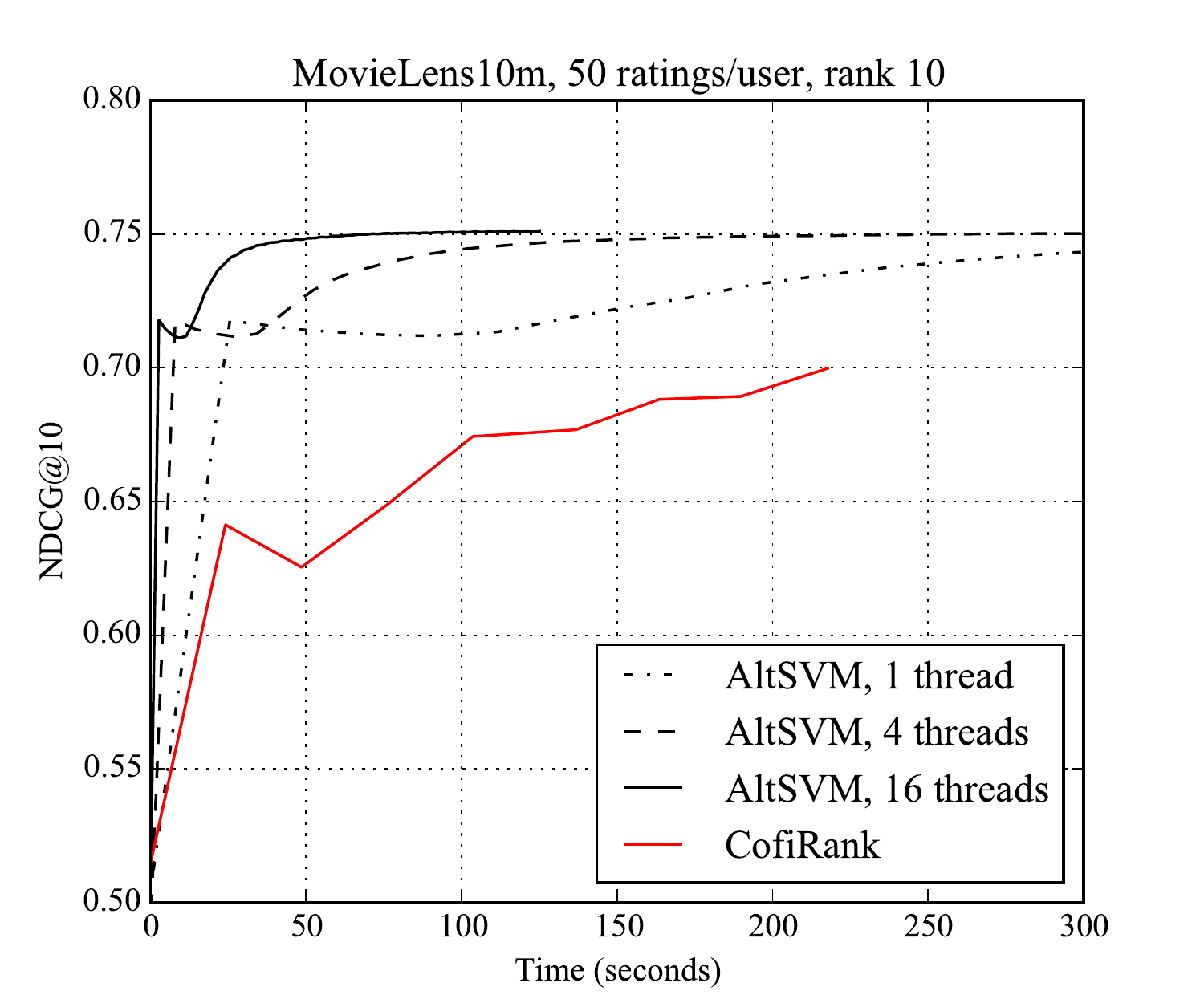}
  \includegraphics[width=.32\textwidth]{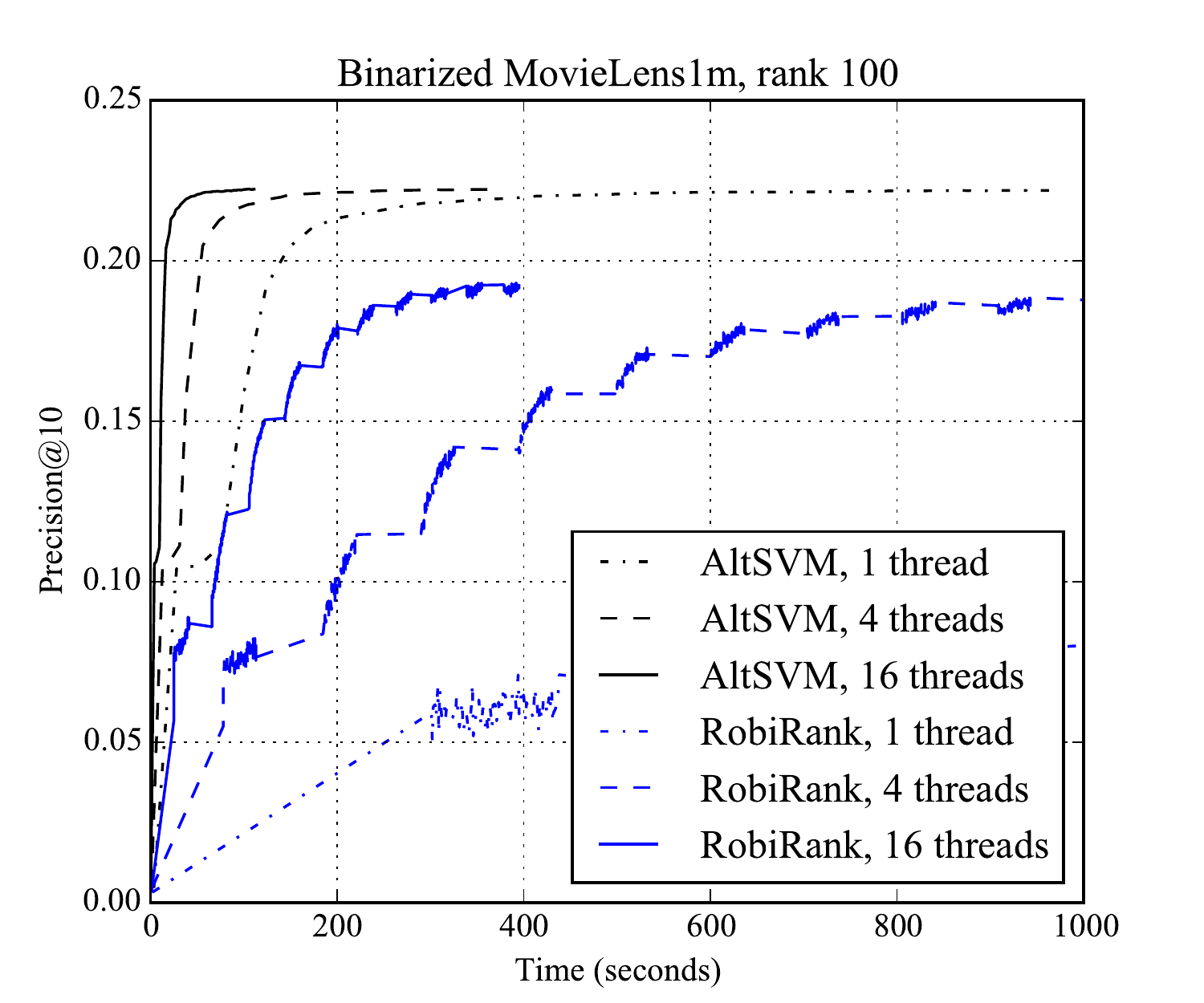}

  (a) \hspace{2.0in} (b) \hspace{2.0in} (c) 
  \caption{NDCG@10 and Precision@10 over time for different algorithms.}
  \label{fig:computation}
\end{center}
\end{figure*}

We first compare our algorithm with numerical rating based algorithms, CofiRank and LCR. We follow the standard setting that are used in the collaborative ranking literature \cite{CofiRank,BalakrishnanChopra,VolkovsZemel,LocalColRanking}. For each user, we subsampled $N$ ratings, used them for training, and took the rest of the ratings for test. The users with less than $N + 10$ ratings were dropped out. Table \ref{tab:NDCG} compares AltSVM with numerical rating based algorithms. While $N=20$ is too small so that a global ranking provides the best NDCG, our algorithm performs the best with larger $N$. We also ran our algorithm with subsampled pairwise comparions with the largest numerical gap (AltSVM-sub), which are as many as $N$ for each user (the number of numerical ratings used in the other algorithms). Even with this, we could achieve better NDCG. We can also observe that the statistical performance is better with the hinge loss than with the logistic loss.
\begin{table*}[t]
\begin{center}
{\scriptsize
\begin{tabular}{|c|c|cccccc|}
\hline
Datasets      & $N$ &   AltSVM & AltSVM-sub & AltSVM-logistic &  Global & CofiRank &   LCR  \\
\hline
              & 20  &   0.7308 &   0.6998 & 0.7125 &\B{0.7500}&   0.7333 &  0.7007 \\
MovieLens1m          & 50  &\B{0.7712}&   0.7392 & 0.7141 &   0.7501 &   0.7441 &  0.7081 \\
              & 100 &\B{0.7902}&   0.7508 & 0.7446 &   0.7482 &   0.7332 &  0.7151 \\
\hline
              & 20  &   0.7059 &   0.7053 & 0.7031 &\B{0.7264}&   0.7076 &  0.6977 \\
MovieLens10m         & 50  &\B{0.7508}&   0.7212 & 0.7115 &   0.7176 &   0.6977 &  0.6940 \\
              & 100 &\B{0.7692}&   0.7248 & 0.7292 &   0.7101 &   0.6754 &  0.6899 \\
\hline
              & 20  &   0.7132 &   0.6822 & - & \B{0.7605}&   0.6615 &    -    \\
Netflix       & 50  &\B{0.7642}&   0.7111 & - & \B{0.7640}&   0.6527 &    -    \\
              & 100 &\B{0.8007}&   0.7393 & - &    0.7656 &   0.6385 &    -    \\
\hline
\end{tabular}
\caption{NDCG@10 on different datasets, for different numbers of observed ratings per user.} \label{tab:NDCG}
}
\end{center}  
\end{table*}
\begin{table}[t]
\begin{center}
{\small
\begin{tabular}{|c|ccc|c|}
\hline
           &          &  AltSVM  &          &   RobiRank \\
Precision@ & $C=1000$ & $C=2000$ & $C=5000$ &            \\
\hline
1   &   0.2165 &   0.2973 &\B{0.3635}&     0.3009 \\
2   &   0.1965 &   0.2657 &\B{0.3297}&     0.2695 \\
5   &   0.1572 &   0.2097 &\B{0.2697}&     0.2300 \\
10  &   0.1265 &   0.1709 &\B{0.2223}&     0.1922 \\
100 &   0.0526 &   0.0678 &\B{0.0819}&     0.0781 \\
\hline
\end{tabular}
\caption{Precision@$K$ on the binarized MovieLens1m dataset.} \label{tab:BinData} 
}
\end{center}  
\end{table}
\begin{table}[t]
\begin{center}
{\small
\begin{tabular}{|c|ccccc|}
\hline
\# cores & 1 & 2 & 4 & 8 & 16 \\
\hline
Time(seconds) & 963.1 & 691.8 & 365.1 & 188.3 & 111.0 \\
\hline
Speedup & 1x & 1.4x & 2.6x & 5.1x & 8.7x \\
\hline
\end{tabular}
\caption{Scalability of AltSVM on the binarized MovieLens1m dataset.} \label{tab:scalability}
}
\end{center}  
\end{table}

We have also experimented with collaborative ranking on binary ratings. We compare 
our algorithm against RobiRank \cite{RobiRank}, which is a recently proposed algorithm for 
collaborative ranking with binary ratings. We ran an experiment on a 
\emph{binarized} version of the Movielens1m dataset. In this case, the movies 
rated by a user is assumed to be relevant to the user, and the other items are 
not. Since it is inefficient to take all possible comparisons which are in 
average a half million per user, we subsampled $C$ comparisons for each user. 
Both algorithms are set to estimate rank-100 matrices. Table \ref{tab:BinData} 
shows that our algorithm provides better performance than RobiRank.

\subsection{Computational speed and Scalability}

We now show the computational speed and scalability of our practical algorithm, AltSVM. The experiments were run on a single 16-core machine in the Stampede Cluster at University of Texas. 
 
Figures \ref{fig:computation}a and \ref{fig:computation}b show NDCG@10 over time of our algorithms with 1, 4, and 16 threads, compared to CofiRank. Figure \ref{fig:computation}c shows Precision@10 over time of our algorithm with $C=5000$. We note that our algorithm converges faster, while the sample size $|\Omega|$ for our algorithm is larger than the number of training ratings that are used in the competing algorithms. Table \ref{tab:scalability} shows the scalability of AltSVM. We measured the time to achieve $10^{-5}$ tolerance on the binarized MovieLens1m dataset. As can be seen in the table, we could achieve significant speedup.

%% file: tex/disc.tex

\section{Conclusion}

We considered the collaborative ranking problem where one fits a low-rank matrix to the pairwise comparisons by multiple users. We showed that the convex relaxation of the empirical risk minimization provides good generalization guarantees. For the large-scale practical settings, we also proposed a non-convex algorithm, which alternately solves two SVM problems. Our algorithm was shown to outperform the existing ones and parallelizes well.

%% file: tex/upper-proof.tex
\section{Proof of Theorem~3.1}
\label{sec:upper}
We write $L(X)$ for the function being optimized; i.e.,
\[
L(X) = \sum_{(i,j,k) \in \Omega} \mathcal{L}(Y_{i,j,k} (X_{i,j} - X_{i,k})).
\]
Note that for any fixed $X$, $\P_{X^*} L(X) = m R(X)$ (where $\P_{X^*}$ denotes the expectation
taken with respect to future samples from $\P_{X^*}$, as distinct from $\E$ which denotes
the expectation over the samples used to generate $\hat X$).
Let $K$ be the set of $d_1 \times d_2$ matrices with nuclear norm at most $1$.
The proof of Theorem~3.1 proceeds in three main steps.
\begin{enumerate}
 \item By some algebraic of manipulations $L$, we reduce the problem to showing a uniform law of large numbers for the family of functions $\{L(X): X \in \sqrt{\lambda d_1 d_2} K\}$.
 \item Using symmetrization and duality properties of $K$, we reduce the problem
 to bounding the norm of a matrix $M$ whose entries are sums of random signs.
 \item We bound the norm of $M$ using various concentration inequalities
 and a theorem of Seginer~\cite{Seginer:00}.
\end{enumerate}

Since $\hat X$, by definition, minimizes $L(\hat X)$, for any $\tilde X \in \sqrt{\lambda d_1 d_2} K$
we can bound
\begin{align*}
 \P_{X^*} [L(\hat X) - L(\tilde X)]
 &\le \P_{X^*} [L(\hat X)] - L(\hat X) - \left(\P_{X^*}[L(\tilde X)] - L(\tilde X)\right) \\
 &\le 2 \sup_{X \in \sqrt{\lambda d_1 d_2} k} |\P_{X^*} L(X) - L(X)|.
\end{align*}
In other words, it suffices to show a uniform law of large numbers
for $\{L(X): X \in \sqrt{\lambda d_1 d_2} K\}$.

Let $\epsilon_{i,j,k}$ be i.i.d. $\pm 1$-valued variables
and let $\xi_{i,j,k}$ be the indicator that $(i, j, k) \in \Omega$.
By Gin\'e-Zinn's symmetrization
(as in~\cite{DPVW:13}),
\begin{align*}
&\sup_{X \in \sqrt{\lambda d_1 d_2} K} |\P_{X^*} L(X) - L(X)| \\
&\le 2 \E \sup_{X \in \sqrt{\lambda d_1 d_2} K} \left|\sum_{i,j,k \in \Omega} \epsilon_{i,j,k}
\mathcal{L}(Y_{i,j,k} (X_{i,j} - X_{i,k}))\right|.
\end{align*}
Since $\mathcal L$ is 1-Lipschitz, we obtain
\begin{align*}
\sup_{X \in \sqrt{\lambda d_1 d_2} K} |\P_{X^*} [L(X)] - L(X)|
&\le 2\E \sup_{X \in \sqrt{\lambda d_1 d_2} K} \left |\sum_{i,j,k \in \Omega} \epsilon_{i,j,k}
Y_{i,j,k} (X_{i,j} - X_{i,k}) \right| \\
&= 2\E \sup_{X \in \sqrt{\lambda d_1 d_2} K} \left |\sum_{i,j,k} \xi_{i,j,k} \epsilon_{i,j,k}
(X_{i,j} - X_{i,k}) \right|,
\end{align*}
where in the last line, we recognized that $\epsilon_{i,j,k} Y_{i,j,k}$ has the
same distribution as $\epsilon_{i,j,k}$.
Now, let $M$ denote the matrix where $M_{ij}
= \sum_k (\xi_{i,j,k} \epsilon_{i,j,k} - \xi_{i,k,j} \epsilon_{i,k,j})$.
Then
\[
\sum_{i,j,k} \xi_{i,j,k} \epsilon_{i,j,k} (X_{i,j} - X_{i,k})
= \tr(M^T X)
\]
and so
\[
\sup_{X \in \sqrt{\lambda d_1 d_2} K} \sum_{i,j,k} \xi_{i,j,k} \epsilon_{i,j,k} (X_{i,j} - X_{i,k})
= \sup_{X \in \sqrt{\lambda d_1 d_2} K} \tr(M^T X) = \sqrt{\lambda d_1 d_2} \|M\|.
\]
Putting everything together, we have (for any $\tilde X \in \sqrt{\lambda d_1 d_2} K$)
\[
\E \left[\P_{X^*}[L(\hat X)] - \P_{X^*}[L(\tilde X)]\right]
\le 4\sqrt{\lambda d_1 d_2} \E \|M\|.
\]
Together with the following lemma (which we prove in Appendix~\ref{app:matrix}),
this completes the proof of Theorem~3.1

\begin{lemma}\label{lem:M-norm} With $p = \frac{m}{d_1 d_2}$,
\[
\E \|M\| \le C \kappa \sqrt{p (d_1 + d_2)} \log (d_1 d_2).
\]
\end{lemma}

%% file: tex/matrix-proof.tex
\section{Proof of Lemma~\ref{lem:M-norm}}
\label{app:matrix}

We will decompose $M$ into two parts, $M = M^{(1)} - M^{(2)}$, with
\begin{align*}
M^{(1)}_{ij} &= \sum_{k \ne j} \xi_{i,j,k} \epsilon_{i,j,k} \\
M^{(2)}_{ij} &= \sum_{k \ne j} \xi_{i,k,j} \epsilon_{i,k,j}.
\end{align*}
Then $\|M\| \le \|M^{(1)}\| + \|M^{(2)}\|$. Since $M^{(1)}$ and $M^{(2)}$ have the same distribution,
\[
\E \|M\| \le 2 \E \| M^{(1)}\|,
\]
and so we are reduced to studying $M^{(1)}$, which has i.i.d.\ entries.
Now, we apply Seginer's theorem~\cite{Seginer:00}:
\begin{equation}\label{eq:seginer}
\E \|M^{(1)}\| \le C \left(\E \max_i \| M^{(1)}_{i*} \|_2 + \E \max_j \| M^{(1)}_{*j}\|_2\right),
\end{equation}
where $M^{(1)}_{i*}$ denotes the $i$th row of $M^{(1)}$ and $M^{(1)}_{*j}$ denotes the $j$th column,
and $\|\cdot\|_2$ denotes the Euclidean norm.

We will separate the task of bounding $\E \max_i \|M_{i*}^{(1)}\|_2$ into two parts:
if $\|x\|_0$ denotes the number of non-zero coordinates in $x$ and $\|x\|_\infty$ denotes
$\max_j |x_j|$ then $\|x\|_2 \le \sqrt{\|x\|_0} \|x\|_\infty$; with the Cauchy-Schwarz inequality,
this implies that
\begin{equation}\label{eq:splitting}
\left(\E \left[\max_i \|M_{i*}^{(1)}\|_2\right]\right)^2
\le \E \left[\max_i \|M^{(1)}_{i*}\|_0\right]
\E \left[\max_i \|M^{(1)}_{i*}\|_\infty^2\right]
\end{equation}

First, we will show that every row of $M^{(1)}$ is sparse.
Let $Z_{ij} = \sum_{k \ne j} \xi_{i,j,k}$ and let $Y_{ij}$ be the indicator that $Z_{ij} > 0$.
Recalling that
$\E \xi_{i,j,k} = p_{i,j,k}$, we have (by Assumption~3.1)
$\E Z_{ij} \le \kappa p$.
Since $Z_{ij}$ takes non-negative integer values, we have
$\Pr(Y_{ij} = 1) = \Pr(Z_{ij} > 0) \le \kappa p$. By Bernstein's inequality,
for any fixed $i$
\[
\Pr(\|M_{i*}^{(1)}\|_0 \ge \kappa d_2 p + t) \le
\Pr(\sum_{j=1}^{d_2} Y_{ij} \ge \kappa d_2 p + t) \le \exp\left(-\frac{t^2/2}{\kappa p d_2 + t/3}\right).
\]
Integrating by parts, we have
\[
\E \left[\|M_{i*}^{(1)}\|_0\right]
\le \kappa d_2 p + \int_{\kappa d_2 p}^\infty \Pr(\|M_{i*}^{(1)}\|_0 \ge t)\ dt \le \kappa d_2 p + \frac 38.
\]

Next, we will consider the size of the elements in $M^{(1)}$.
First of all, $M^{(1)}_{ij} \le Z_{ij}$ (this fairly crude bound
will lose us a factor of $\sqrt{\log(d_1 d_2)}$). Now, Bernstein's
inequality applied to $Z_{ij}$ gives
\[
\Pr(M_{ij}^{(1)} \ge \kappa p + t) \le
\Pr(Z_{ij} \ge \kappa p + t) \le \exp\left(-\frac{t^2/2}{\kappa p + t/3}\right).
\]
Taking a union bound over $i$ and $j$, if $t \ge C \kappa \log(d_1 d_2)$ then
\[
\Pr(\max_{ij} M^{(1)}_{ij} \ge t) \le d_1 d_2 \exp\left(-ct\right)
\le \exp(-c' t).
\]
Integrating by parts,
\[
\E \left[ \max_{ij} M^{(1)}_{ij} \right] \le \kappa \log^2(d_1 d_2) + \int_{\kappa \log^2(d_1 d_2)}^\infty
\Pr(\max_{ij} M^{(1)}_{ij} \ge \sqrt{t}) \ dt
\le \kappa \log^2(d_1 d_2) + C.
\]
Going back to~\eqref{eq:splitting}, we have shown that
\[
\E \max_i \|M^{(1)}_{i*}\| \le C \kappa \sqrt{p d_2} \log(d_1 d_2).
\]
The same argument applies to $M^{(1)}_{*j}$ (but with $\sqrt{p d_1}$ instead
of $\sqrt{p d_2}$), and so we conclude from~\eqref{eq:seginer} that
\[
\E \|M^{(1)}\| \le C \kappa \sqrt{p (d_1 + d_2)} \log (d_1 d_2).
\]

%% file: tex/lower-sketch.tex
\section{Proof of Theorem~3.2}
\label{sec:lower}

\subsection{A sketch of the proof} 
The proof of Theorem~3.2 uses Fano's inequality.
\begin{enumerate}
 \item We construct matrices $X^1, \dots, X^\ell$. These matrices all have
 small nuclear norm, and for every pair $i, j$ the KL-divergence between the induced
 observation distributions is $\Theta(\log \ell)$. We construct these matrices randomly,
 using concentration inequalities and a union bound to show that we can take $\ell$
 of the order $\sqrt{\lambda m (d_1 + d_2)}$.
 
 \item We apply Fano's inequality to show that if we generate data according to a
 randomly chosen $X^i$, then any algorithm has a reasonable chance to choose a different
 $X^j$ (using the fact that the KL-divergence is $O(\log \ell)$). Since
 the KL-divergence is $\Omega(\log \ell)$, this implies that the algorithm incurs a substantial
 penalty whenever it makes a wrong choice.
\end{enumerate}

In any application of Fano's inequality, the key is to construct a large number of
admissible models that are close to one another in KL-divergence. Specifically, if we can construct
distributions $\P_1, \dots, \P_\ell$ with $D(\P_i \| \P_j) + 1 \le \frac{1}{2} \log \ell$ for all $i, j$,
then given a single sample from some $\P_i$, no algorithm can accurately identify which $\P_i$ it
came from. In order to apply this denote by $\P_{X,m}$ the distribution of the data
when the true parameters are $X$.
We will construct $X^1 \dots, X^\ell \in \sqrt{\lambda d_1 d_2} K$ such that
for all $i \ne j$,
\begin{align}
D(\P_{X^i,m} \| \P_{X^j,m}) + 1 \le \frac 12 \log \ell, \label{eq:fano-requirement} \\
R_j(X^i) \ge R_j(X^j) + c \frac{\log \ell}{m} \label{eq:separation}
\end{align}
for some constant $c > 0$,
where $R_j$ denotes the expected risk when the true parameters are given by $X^j$.
Given a single observation from some $\P_{X^j,m}$,~\eqref{eq:fano-requirement}
will imply (by Fano's inequality) that no algorithm can correctly
identify which $X^j$ was the true parameter. On the other hand,~\eqref{eq:separation} will imply
that if the algorithm makes a mistake -- say it chooses $X^i$ for $i \ne j$ -- then its risk
will be $c \frac{\log \ell}{m}$ larger than the best in the class.
In particular, if we can prove~\eqref{eq:fano-requirement} and~\eqref{eq:separation}
with $\log \ell \sim \sqrt{\lambda m(d_1 + d_2)}$ then it will imply Theorem~3.2.

We construct a set of matrices satisfying~\eqref{eq:fano-requirement}
and~\eqref{eq:separation} using a probabilistic method.
Supposing that $d_2 \ge d_1$, we choose a parameter $\gamma > 0$ and set $B$ to be an integer
that is approximately $\lambda \gamma^{-2}$.
We define $X^1$ by filling its top $B \times d_2$ block with independent, uniform $\pm \gamma$ entries,
and then copying that top block $B/d_1$ times to fill the matrix. Then let
$X^2, \dots, X^\ell$ be independent copies of $X^1$.
First of all, each $X^i \in \sqrt{\lambda d_1 d_2} K$ because $\|X^i\|_* \le \sqrt{\rank(X^i)} \|X^i\|_F \le \sqrt{\lambda d_1 d_2}$.

Now, let us consider $D(\P_{X^1,m} \| \P_{X^2,m})$.
For a single $i,j,k$ triple, there is probability $1/4$ of having
$X^1_{i,j} - X^1_{i,k}$ different from $X^2_{i,j} - X^2_{i,k}$, in which case they differ by
$4 \gamma$. If $\gamma$ is bounded above, each different entry contributes
$\Theta(\alpha^2 \gamma^2)$ to the KL-divergence between $\P_{X^1,m}$ and $\P_{X^2,m}$. Since
about $m$ entries are observed in $\P_{X^1,m}$, we see that
\begin{equation}\label{eq:logl}
D(\P_{X^1,m} \| \P_{X^2,m}) \asymp m \gamma^2.
\end{equation}
On the other hand, $R_1(X^1)$ and $R_1(X^2)$ differ by $\Theta(\gamma^2)$, because
for a constant fraction of triples $i, j, k$, the chance that $Y_{i,j,k}$ is 1 differs by
$O(\gamma)$ in $X^1$ and $X^2$, and on the event that $Y_{i,j,k}$ differs in these two models
the loss differs by another $O(\gamma)$ factor.

Applying standard concentration inequalities, we show that one can apply the union bound
to $\ell = \exp(c B d_2)$ of these matrices. In view of~\eqref{eq:fano-requirement} and~\eqref{eq:logl},
we need to take $B d_2  = \frac{\lambda^2}{\gamma^2 d_1} \asymp m \gamma^2$. Eliminating $\gamma$,
we end up with $\log \ell \asymp \sqrt{\lambda m/d_1}$ (which is within a constant factor
of $\sqrt{\lambda m(d_1 + d_2)}$ under our assumption that $d_2 \ge d_1$).

%% file: tex/lower-proof.tex
\subsection{Some concentration lemmas}

We begin by quoting some standard concentration results (see, e.g.~\cite{Vershynin:12}).

\begin{definition}
A random variable $X$ is $\sigma^2$-subgaussian if $\E e^{\theta X} \le e^{\theta^2 \sigma^2/2}$ for all $\theta > 0$.
A random variable
$X$ is $L$-subexponential if $\E e^{\theta X} \le (1 - \theta^2 L^2)$ for $\theta < 1/L$.
\end{definition}

One can easily show that the product of two subgaussian variables is subexponential:

\begin{lemma}\label{lem:subexponential}
If $X$ is $\sigma^2$-subgaussian and $Y$ is $\tau^2$-subgaussian then
$XY$ is $C \sigma \tau$-subexponential for a universal constant $C$.
\end{lemma}

Moreover, one has a Bernstein-type inequality for sums of independent subexponential
variables.

\begin{lemma}\label{lem:bernstein}
If $X_1, \dots, X_k$ are i.i.d. $L$-subexponential then
\[
  \Pr(\sum_i X_i \ge t) \le \exp\left(-\frac{c t^2}{L^2 k + L t}\right).
\]
\end{lemma}

\subsection{Construction of a packing set}

Let $0 < \gamma < 1$ be some parameter to be determined such that $B := \lambda \gamma^{-2}$ is an
integer.

\begin{proposition}\label{prop:packing-set}
Suppose that $\mathcal{L}'(0) < 0$.
For every sufficiently small $\gamma$ (depending on $\mathcal{L}$),
there exists a set $\calX \subset \sqrt{\lambda d_1 d_2} K$ of $\exp(c B d_2)$ $d_1 \times d_2$ matrices such that
for any two $X^1, X^2 \in \calX$,
\[
  \frac{1}{d_1 d_2^2} \sum_{i=1}^{d_1} \sum_{j,k=1}^{d_2}
  \E_{X^1}[
  \mathcal{L}(Y(X^2_{ij} - X^2_{ik}))
  - \mathcal{L}(Y(X^1_{ij} - X^1_{ik}))
  ] \ge c \gamma^2
\]
and for any $m$,
\[
  \frac{1}{m} D(\P_{X^1,m} \| \P_{X^2,m}) \le C \gamma^2,
\]
where $0 < c < C$ are universal constants.
\end{proposition}

Following Davenport et al., we construct this set $\calX$ randomly:
let $X$ be a random $B \times d_2$ matrix, where each element is chosen independently to be
either $\gamma$ or $-\gamma$.

\begin{lemma}\label{lem:concentrated-F-norm}
Let $X^1$ and $X^2$ be independent copies of $X$.
Then with probability at least $1-\exp(-c B d_2)$,
\[
  \sum_{i=1}^{B} \sum_{j,k = 1}^{d_2} (X^1_{ij} - X^1_{ik} - X^2_{ij} + X^2_{ik})^2 \ge 2 \gamma^2 B d_2^2,
\]
where $c > 0$ is a universal constant.
\end{lemma}

Before proving Lemma~\ref{lem:concentrated-F-norm}, let us see how it implies
Proposition~\ref{prop:packing-set}. First of all, for $X$ a random $B \times d_2$ matrix as above,
let $\tilde X$ be the $d_1 \times d_2$ matrix obtained by stacking $\lceil d_1/B \rceil$ copies
of $X$, and filling out any remaining entries by zeros. Then, for random $X$ and $Y$,
with high probability
\begin{align}
  \sum_{i=1}^{d_1} \sum_{j,k = 1}^{d_2} (\tilde X^1_{ij} - \tilde X^1_{ik} - \tilde X^2_{ij} + \tilde X^2_{ik})^2
  &= \lceil d_1 / B \rceil \sum_{i=1}^{B} \sum_{j,k = 1}^{d_2} (X^1_{ij} - X^1_{ik} - X^2_{ij} + X^2_{ik})^2 \notag \\
  &\asymp \gamma^2 d_1 d_2^2,\label{eq:F-norm}
\end{align}
where the lower bound for the last line came from Lemma~\ref{lem:concentrated-F-norm}, and the upper bound
just came from the observation that each term in the sum is bounded by $16 \gamma^2$.
Let $\calX$ be the set obtained by choosing $\exp(c B d_2 / 4)$ random copies of $\tilde X$ in this way.
The high-probability estimate in Lemma~\ref{lem:concentrated-F-norm} implies that with high probability,
\emph{every} pair $\tilde X^1, \tilde X^2$ in $\calX$ satisfies~\eqref{eq:F-norm}.
Now,
\begin{align*}
D(\P_{X^1,m} \| \P_{X^2,m})
&= \E_\Omega \left[\sum_{(i,j,k) \in \Omega} D(f(X^1_{ij} - X^1_{ik}) \| f(X^2_{ij} - X^2_{ik}))\right] \\
&\asymp \frac{m}{d_1 d_2^2} \sum_{i,j,k} (X^1_{ij} - X^1_{ik} - X^2_{ij} + X^2_{ik})^2,
\end{align*}
where $f(x) = e^{x} / (1 + e^x)$ is the logistic function, and the last line follows from a Taylor expansion
of $D(f(x) \| f(y))$ around $x = y$, because all the
$X^1_{ij}$ and $X^2_{ij}$ are bounded by $\gamma < 1$. Together with~\eqref{eq:F-norm},
this proves the first inequality in Proposition~\ref{prop:packing-set}; the second
inequality follows because each term of the form $D(f(X_{ij} - X_{ik}) \| f(Y_{ij} - Y_{ik}))$ is bounded
by a constant times $\gamma^2$.
This proves the second inequality of Proposition~\ref{prop:packing-set}.

By Taylor expansion again, if $\gamma$ is sufficiently small (depending on $\mathcal{L}$) then
\[
\mathcal{L} (Y_{i,j,k}(X^2_{i,j} - X^2_{i,k}))
- \mathcal{L} (Y_{i,j,k}(X^1_{i,j} - X^1_{i,k}))
\asymp Y_{i,j,k} (X^1_{i,j} - X^1_{i,k} - X^2_{i,j} + X^2_{i,k}).
\]
Now, if $i,j,k$ is a triple for which $2\gamma = X^1_{i,j} - X^1_{i,k} > X^2_{i,j} - X^2_{i,k}$
(and under the event of Lemma~\ref{lem:concentrated-F-norm}, there are at least $c B d_2^2$ such
triples) then $\E_{X^1} [Y_{i,j,k}] \asymp \gamma$ and so
\[
\E_{X^1}[\mathcal{L} (Y_{i,j,k}(X^2_{i,j} - X^2_{i,k}))
- \mathcal{L} (Y_{i,j,k}(X^1_{i,j} - X^1_{i,k}))] \asymp \gamma^2.
\]
The same holds when $i,j,k$ is a triple for which
$-2\gamma = X^1_{i,j} - X^1_{i,k} < X^2_{i,j} - X^2_{i,k}$. Finally, if $i,j,k$ is a triple
such that $X^1_{i,j} - X^1_{i,k} = X^2_{i,j} - X^2_{i,k}$ then the expectation is zero.
Summing over all triples, we see that on the event that Lemma~\ref{lem:concentrated-F-norm} holds,
\[
\frac{1}{B d_2^2}
\sum_{i,j,k} \E_{X^1}[\mathcal{L} (Y_{i,j,k}(X^2_{i,j} - X^2_{i,k})) - \mathcal{L} (Y_{i,j,k}(X^1_{i,j} - X^1_{i,k}))] \ge c \gamma^2.
\]
After summing over all $\lceil d_1/B \rceil$ blocks, this proves the first inequality 
of Proposition~\ref{prop:packing-set}.

\begin{proof}[Proof of Lemma~\ref{lem:concentrated-F-norm}]
We expand the square:
\begin{align}
  \sum_{ijk} (X_{ij} - X_{ik} - Y_{ij} + Y_{ik})^2
  &= 2 \sum_{ijk} X_{ij}^2 + Y_{ij}^2 + 2 X_{ij} Y_{ik} - X_{ij} X_{ik} - Y_{ij} Y_{ik} - 2 X_{ij} Y_{ij} \notag \\
  &= 4 \gamma^2 B d_2^2 + 2 \sum_{ijk} 2 X_{ij} Y_{ik} - X_{ij} X_{ik} - Y_{ij} Y_{ik} - 2 X_{ij} Y_{ij}.
  \label{eq:expand-square}
\end{align}
We may study each of the cross-terms separately: for the $X_{ij} Y_{ik}$ term,
note that $\sum_{j} X_{ij}$ and $\sum_k Y_{ik}$ are both $\gamma^2 d_2$-subgaussian (by Hoeffding's inequality).
Hence, $\sum_{jk} X_{ij} Y_{ik}$ is $C \gamma^2 d_2$-subexponential (by Lemma~\ref{lem:subexponential}) and so
by Lemma~\ref{lem:bernstein},
\[
  \Pr\left(\left|\sum_{ijk} X_{ij} Y_{ik}\right| \ge \frac 18 \gamma^2 B d_2^2 \right) \le 2 \exp(- c B d_2).
\]
The similar argument applies to the $X_{ij} X_{ik}$ term: $\sum_j X_{ij}$ is
$\gamma^2 d_2$-subgaussian and so $\sum_{ijk} X_{ij} X_{ik} = \sum_i (\sum_j X_{ij})^2$ is
$C \gamma^2 d_2$-subexponential; hence
\[
  \Pr\left(\left|\sum_{ijk} X_{ij} X_{ik}\right| \ge \frac 18 \gamma^2 B d_2^2 \right) \le 2 \exp(- c B d_2).
\]
Of course, the $Y_{ij} Y_{ik}$ term is identical. Finally, note that
$\sum_{ijk} X_{ij} Y_{ij} = d_2 \sum_{ij} X_{ij} Y_{ij}$. Since the terms in this sum are i.i.d., we may apply
Hoeffding's inequality to obtain
\[
  \Pr\left(\left|\sum_{ijk} X_{ij} Y_{ij}\right| \ge \frac 18 \gamma^2 B d_2^2 \right)
  = \Pr\left(\left|\sum_{ij} X_{ij} Y_{ij}\right| \ge \frac 18 \gamma^2 B d_2 \right)
  \le 2 \exp(- c B^2 d_2^2).
\]
Putting everything together, we see that with high probability, the total of all the cross-terms in~\eqref{eq:expand-square}
is at most half of the first term.
\end{proof}

\subsection{Completing the proof}
Let $C$ denote the constant from Proposition~\ref{prop:packing-set}.
Assume that $d_1 \le d_2$ and that $m$ is large enough so
\begin{equation}\label{eq:B-range}
\sqrt{\frac{d_2}{m}} \le 8 C \sqrt \lambda \le \sqrt{\frac{m}{d_2}}.
\end{equation}
Note that under the assumptions $\lambda \ge 1$ and $m \ge d_1 + d_2$ from
Theorem~3.2, the lower bound of~\eqref{eq:B-range} is satisfied. Moreover,
if the upper bound of~\eqref{eq:B-range} is not satisfied then we may decrease $\lambda$
until it is; the conclusion of Theorem~3.2 will not be affected because
as long as~\eqref{eq:B-range} fails, the minimum in Theorem~3.2 will be 1.

By the lower bound in~\eqref{eq:B-range}, there is an integer $B$ such that
\[
B \le \sqrt{\frac{\lambda m}{d_2}} \le 2B;
\]
fix this $B$ and define $\gamma$ by
\[
\gamma^2 = \lambda / B \asymp \sqrt{\frac{\lambda d_2}{m}}.
\]
By the upper bound in~\eqref{eq:B-range}, $\gamma \le 1$.

Now, Fano's inequality states that if we first select a random $X \in \calX$ and then
draw a sample from $\P_{X, m}$, then any algorithm trying to identify $X$ can succeed with probability
at most
\[
\frac{\min\{D(\P_{X, m} \| \P(Y, m)): X, Y \in \calX\} + 1}{\log |\calX|}
\le \frac{2 C m \gamma^2}{B d_2} \le \frac 12.
\]
Finally, note that by the first inequality in Proposition~\ref{prop:packing-set},
the error incurred by choosing the wrong $X \in \calX$ is at least $c \gamma^2
\asymp \sqrt{\frac{\lambda d_2}{m}}$.

Now, we have so far only discussed the case $d_2 \ge d_1$. The case $d_1 \le d_2$ is not
exactly equivalent because our model is not symmetric in its treatment of users
and items. However, the proof of Theorem~3.2 does not change very much. We
take horizontally stacked blocks of size $d_1 \times B$ instead of $B \times d_2$. The
main difference is in the calculation leading to~\eqref{eq:F-norm}: there are extra
cross-terms appearing due to the fact that items in different blocks need to
be compared with one another. However, all of these additional terms may be
controlled with Lemmas~\ref{lem:subexponential} and~\ref{lem:bernstein} in
much the same way as the existing terms are controlled.

%% file: tex/sgd.tex
\section{Comparison to Stochastic Gradient Descent}
Another practical algorithm to optimize (3) is Stochastic Gradient Descent (SGD). We have experimented SGD on the same datasets in Table 1. We ran the algorithm with the same regularization parameters and different step sizes. The statistical results for SGD were observed to be no better than AltSVM, and hence we did not present them in the main paper.

Let us first describe the SGD procedure. At each step, ones chooses a triple $(i,j,k) \in \Omega$ uniformly at random and run a SGD step, which can be written as
\begin{align*}
u_i^+ &\leftarrow u_i - \eta \cdot \left\{ g \cdot (v_j - v_k) + \frac{\lambda}{|\Omega_i|} u_i \right\} \\
v_j^+ &\leftarrow v_j - \eta \cdot \left\{ g \cdot u_i + \frac{\lambda}{|\Omega^{j}|} v_j \right\} \\
v_j^+ &\leftarrow v_j - \eta \cdot \left\{ - g \cdot u_i + \frac{\lambda}{|\Omega^{k}|} v_k \right\}
\end{align*}
where $\Omega^{(j)}$ denotes the number of comparisons in $\Omega$ which involve item $j$. $\eta$ is a step size and $g \in \partial \mathcal{L}(u_i^\top (v_j - v_k))$.

The following tables show the statistical result of SGD. The step size is chosen by $\eta = \frac{\alpha}{1 + \beta t}$ as suggested in \cite{NOMAD}. $\alpha$ and $\beta$ were the powers of $10^{-1}$, and the best result is reported. The results are comparable to AltSVM, but it did not achieve better results. We note that this is the best result from several different step sizes, while AltSVM does not have any other parameter to choose except for the regularization parameter.
\begin{table}[t]
\begin{center}
{\small
\begin{tabular}{|c|c|c|}
\hline
Datasets      & $N$ & NDCG@10 \\
\hline
              & 20  & 0.6852 \\
ML1m          & 50  & 0.7666 \\
              & 100 & 0.7728 \\
\hline
              & 20  & 0.6977 \\
ML10m         & 50  & 0.7452 \\
              & 100 & 0.7659 \\
\hline
\end{tabular} \label{tab:SGD}
\caption{NDCG@10 of SGD on different datasets, for different numbers of observed ratings per user.}
}
\end{center}  
\end{table}

\begin{table}[t]
\begin{center}
{\small
\begin{tabular}{|c|c|}
\hline
Precision@ & SGD with $C=5000$ \\
\hline
1   &   0.1556 \\
2   &   0.1498 \\
5   &   0.1236 \\
10  &   0.1031 \\
100 &   0.0441 \\
\hline
\end{tabular}
\caption{Precision@$K$ for SGD of (3) on the binarized MovieLens1m dataset.} \label{tab:SGD2} 
}
\end{center}  
\end{table}